\pdfoutput=1

\documentclass[twocolumn,oneside,fontsize=9pt,paper=a4,pagesize,DIV=calc]{scrartcl}
\RequirePackage{etoolbox}

\usepackage{hyperref}

\usepackage{times}
\usepackage[dvipsnames]{xcolor}
\usepackage{float}
\usepackage[caption=false]{subfig}
\usepackage{graphicx}
\usepackage{bibentry}
\usepackage{algorithm}
\usepackage[noend]{algorithmic}
\usepackage{booktabs}
\usepackage{multirow}
\usepackage{amsmath,amsfonts,amsthm}
\usepackage[capitalize]{cleveref}
\usepackage[inline,shortlabels]{enumitem}
\usepackage{calc}
\usepackage{authblk}

\AfterEndPreamble{\let\ref\cref\undef{\cref}\let\Ref\Cref\undef{\Cref}}
\Crefname{equation}{}{}\crefname{equation}{}{}
\Crefname{myprob}{Problem}{Problems}\crefname{myprob}{Problem}{Problems}
\theoremstyle{plain}

\newtheorem{mylemma}{Lemma}
\newtheorem{mythm}{Theorem}
\theoremstyle{definition}
\newtheorem{myprob}{Problem}
\newcommand{\tabformathead}[1]{\textbf{\boldmath{#1}}}

\newenvironment{mtables}[1]{\begin{table}[t]\caption{#1}\begin{center}\begin{tiny}}{\end{tiny}\end{center}\end{table}}
\newcommand{\eigvec}[1]{\psi_{#1}}
\newcommand{\eigveca}[1]{a_{#1}}
\newcommand{\eigvecv}[1]{v_{#1}}
\newcommand{\eigval}[1]{\lambda_{#1}}
\newcommand{\showline}[1]{[{\color{#1}\rule[2pt]{5pt}{2pt}}]}
\newcommand{\showsymb}[2]{[{\color{#1}#2}]}
\newcommand{\showaster}[1]{\showsymb{#1}{*}}
\colorlet{coltriv}{Goldenrod!70!Black}%
\colorlet{colntriv}{NavyBlue}%
\colorlet{colnboun}{BrickRed}%
\colorlet{colfill}{SeaGreen!20}%
\DeclareMathOperator{\spano}{span}
\newcommand{\spans}[1]{\spano\prn{#1}}
\DeclareMathOperator{\sign}{sign}
\let\originalleft\left
\let\originalright\right
\renewcommand{\left}{\mathopen{}\mathclose\bgroup\originalleft}
\renewcommand{\right}{\aftergroup\egroup\originalright}
\newcommand{\prn}[1]{\left ( #1 \right )}
\newcommand{\brc}[1]{\left \{ #1 \right \}}
\newcommand{\brq}[1]{\left [ #1 \right ]}
\newcommand{\set}[1]{\brc{#1}}
\newcommand{\R}{\mathbb{R}}
\newcommand{\eeq}[1]{\; #1}
\newcommand{\inp}[2]{\langle #1 | #2 \rangle}
\newcommand{\inpt}[3]{\langle #1 | #2 | #3 \rangle}
\newcommand{\inpe}[2]{\langle #1 | #2 \rangle_{\Esp}}
\newcommand{\bra}[1]{\langle #1 |}
\newcommand{\ket}[1]{| #1 \rangle}
\newcommand{\sepeq}{\eeq{,} \quad}
\newcommand{\sepfor}{\quad\text{for }}
\newcommand{\proj}{\mathcal{P}}
\newcommand{\opt}[1]{#1^\star}
\newcommand{\tr}{\top}
\newcommand{\Xset}{\mathcal{X}}
\newcommand{\Hsp}{\mathcal{H}}
\newcommand{\Esp}{\mathcal{E}}
\newcommand{\fmap}{\varphi}
\newcommand{\fmapx}[1]{\fmap\prn{#1}}
\newcommand{\covard}{\Phi \circ \Phi^*}
\newcommand{\kernelp}{\Phi^* \circ \Phi}
\newcommand{\eb}{\mathsf{e}}
\newcommand{\agropt}[1]{#1}
\newcommand{\sepopt}{~}

\newcommand{\minp}[2]{\min_{#1}\sepopt{\agropt{#2}}}

\newcommand{\minpcl}[3]{\left \{\begin{array}{l}\displaystyle \min_{#1}\sepopt{\agropt{#2}}\\\displaystyle\text{s.t. }#3\end{array} \right .}
\newcommand{\minpcla}[3]{\left \{\begin{array}{l}\displaystyle \min_{#1}\sepopt{\agropt{#2}}\\\displaystyle\text{s.t. }\left \{ \begin{array}{l}#3\end{array} \right . \end{array} \right .}
\newcommand{\lagr}{\mathcal{L}}
\newcommand{\svm}{Subs-LSSVM}
\newcommand{\ssvm}{Semi-LSSVM}
\newcommand{\ssvmsvm}{Semi/Subs-LSSVM}
\newcommand{\kpca}{KPCA}
\newcommand{\skpca}{Semi-KPCA}
\newcommand{\skpcak}[1]{\skpca$_{#1}$}

\newif\iftikz
\newcommand{\tikzpath}{./Figures/Tikz}
\newcommand{\pdfpath}{./}

\iftikz\fi

\newcounter{Fig}

\newcommand{\includetikz}[1]{\iftikz\begin{minipage}[b]{\figwidth}\vspace{0pt}\scriptsize\centering\tikzincexp\tikzsetnextfilename{\tikznamee{#1}}\input{\tikzpath/#1.tikz}\end{minipage}\else\includetikzpdf{#1}\fi}
\newcommand{\includetikzf}[1]{\iftikz\includetikzextkz{#1}\else\includetikzexpdf{#1}\fi}

\newcommand{\tikznameexp}[1]{#1\tikzsuf}
\newcommand{\tikznameseq}[1]{\tikzbase\theFig\tikzsuf}

\newcommand{\includetikzpdf}[1]{\centering\tikzincexp\includegraphics{\pdfpath/\tikznamee{#1}.pdf}}
\newcommand{\includetikzextkz}[1]{\begin{minipage}[b]{\figwidth}\vspace{0pt}\scriptsize\centering\tikzincseq\tikzsetnextfilename{\tikznames{#1}}{#1}\end{minipage}}
\newcommand{\includetikzexpdf}[1]{\centering\tikzincseq\includegraphics{\pdfpath/\tikznames{#1}.pdf}}
\newcommand{\figwidth}{\columnwidth}
\newcommand{\tikzwidth}[1]{\renewcommand{\figwidth}{#1}}

\newcommand{\tikzincseq}{\stepcounter{Fig}}
\newcommand{\tikzincexp}{}

\newcommand{\tikzbase}{}
\newcommand{\tikzsuf}{}
\newcommand{\tikznamee}[1]{\tikznameexp{#1}}
\newcommand{\tikznames}[1]{\tikznameseq{#1}}

\newcommand{\tikzmodearxiv}{\renewcommand{\tikzbase}{Fig}\renewcommand{\tikzsuf}{_arXiv}\renewcommand{\tikznamee}[1]{\tikznameexp{##1}}\renewcommand{\tikznames}[1]{\tikznameseq{##1}}\setcounter{Fig}{0}}

\newcommand{\mycaption}{}
\newenvironment{mfiguretikz}[1]{\renewcommand{\mycaption}{#1}\begin{figure*}[t]\begin{center}\begin{minipage}{0.9\textwidth}}{\end{minipage}\caption{\mycaption}\end{center}\end{figure*}}
\newenvironment{mfiguretikzs}[1]{\renewcommand{\mycaption}{#1}\begin{figure}[t]\begin{center}\begin{minipage}{0.4\textwidth}\begin{center}}{\end{center}\end{minipage}\caption[Caption]{\mycaption}\end{center}\end{figure}}

\newenvironment{malgorithm}[1]{\begin{algorithm}[t]\caption{#1}\begin{center}\begin{scriptsize}}{\end{scriptsize}\end{center}\end{algorithm}}

\newcommand{\algstop}{;}
\newcommand{\algreq}{$\cdot$~}
\algsetup{linenosize=\tiny}

\newcommand{\formata}[1]{{\colorbox{ForestGreen}{\color{Black}\makebox[\widthof{$99.9\pm99.9$}][c]{\raisebox{0pt}[\height][0pt]{#1}}}}}
\newcommand{\formatb}[1]{{\colorbox{ForestGreen!50}{\color{Black}\makebox[\widthof{$99.9\pm99.9$}][c]{\raisebox{0pt}[\height][0pt]{#1}}}}}
\newcommand{\formatc}[1]{{\colorbox{ForestGreen!20}{\color{Black}\makebox[\widthof{$99.9\pm99.9$}][c]{\raisebox{0pt}[\height][0pt]{#1}}}}}
\makeatletter
\newcommand{\showemail}[1]{Email: \href{mailto:#1@esat.kuleuven.be}{#1@esat.kuleuven.be}.}
\makeatother

\title{Convex Formulation for Kernel PCA and its Use in Semi-Supervised Learning}

\author{Carlos M. Ala\'iz \thanks{\showemail{cmalaiz}}}
\author{Micha\"el Fanuel \thanks{\showemail{michael.fanuel}}}
\author{Johan A. K. Suykens \thanks{\showemail{johan.suykens}}}
\affil{KU Leuven, ESAT, STADIUS Center. B-3001 Leuven, Belgium.}

\date{\today}
\tikzmodearxiv

\begin{document} 

\maketitle

%%%%%%%%%%%%%%%%%%%%%%%%%%%%%%%%%%%%%%%%%%%%%
\begin{abstract}
 In this paper, Kernel PCA is reinterpreted as the solution to a convex optimization problem. Actually, there is a constrained convex problem for each principal component, so that the constraints guarantee that the principal component is indeed a solution, and not a mere saddle point.
 Although these insights do not imply any algorithmic improvement, they can be used to further understand the method, formulate possible extensions and properly address them.
 As an example, a new convex optimization problem for semi-supervised classification is proposed, which seems particularly well-suited whenever the number of known labels is small. Our formulation resembles a Least Squares SVM problem with a regularization parameter multiplied by a negative sign, combined with a variational principle for Kernel PCA. Our primal optimization principle for semi-supervised learning is solved in terms of the Lagrange multipliers.
 Numerical experiments in several classification tasks illustrate the performance of the proposed model in problems with only a few labeled data.
\end{abstract}
%%%%%%%%%%%%%%%%%%%%%%%%%%%%%%%%%%%%%%%%%%%%%

%%%%%%%%%%%%%%%%%%%%%%%%%%%%%%%%%%%%%%%%%%%%%
\section{Introduction}
\label{SecIntro}

Kernel PCA~(\kpca{};~\cite{Schoelkopf_KPCA}) is a well-known unsupervised method which has been used extensively, for instance in novelty detection~\cite{Novelty_Dectection} or image denoising~\cite{Mika:1999}. As for PCA, there are many different interpretations of this method. In this paper, motivated by the importance of convex optimization in unsupervised and supervised methods, we firstly study the following question: ``to which convex optimization problem \kpca{} is the solution?''. Hence, we find that there exists a constrained convex optimization problem for each principal component of \kpca{}. A major asset of our formulation with respect to previous studies in the literature~\cite{IEEE_Suykens} is that the optimization problem is proved to be convex.
Moreover, this result helps us to unravel the connection between unsupervised and supervised methods. We formulate a new convex problem for semi-supervised classification, which is a hybrid formulation between \kpca{} and Least Squares SVM (LS-SVM;~\cite{BookSuykens}).

We can summarize the theoretical and empirical contributions of this paper as follows:
\begin{enumerate*}[(i)]
 \item we define a convex optimization problem for \kpca{} in a general setting including the case of  an infinite dimensional feature map;
 \item a variant of \kpca{} for semi-supervised classification is defined as a constrained convex optimization problem, which can be interpreted as a regression problem with a concave error function, the regularization constant being constrained so that the problem remains convex; and
 \item the latter method is studied empirically in several classification tasks, showing that the performance is particularly good when only a small number of labels is known.
\end{enumerate*}

The paper is structured as follows. In \ref{SecVarPrinc} we review \kpca{} and LS-SVM. In \ref{SecConvexOpt} we introduce the convex formulation of \kpca{}, which is extended to semi-supervised learning in \ref{SecSemiSup}. \Ref{SecExamples} includes the numerical experiments, and we end with some conclusions in \ref{SecConcl}.
%%%%%%%%%%%%%%%%%%%%%%%%%%%%%%%%%%%%%%%%%%%%%

%%%%%%%%%%%%%%%%%%%%%%%%%%%%%%%%%%%%%%%%%%%%%
\section{Variational Principles}
\label{SecVarPrinc}

Let us suppose that the data points $\set{x_i}_{i=1}^{N}$ are in a set $\Xset$, which can be chosen to be $\R^d$, but can also be a set of web pages or text documents, for instance. The data are embedded in the feature space $\Hsp$ with the help of a feature map. In this introductory section, we will choose $\Hsp = \R^h$.

Given a set of pairs of data points and labels $\set{\prn{x_i,y_i}}_{i=1}^{N}$ where $y_i = \pm 1$ is the class label of the point $x_i \in \Xset$, the Least Squares SVM (LS-SVM) problem, which is convex with a unique solution, is formulated as follows.
\begin{myprob}[LS-SVM for Regression] Let $w \in \R^h$ and $e_i \in \R$, while $\gamma > 0$ is a regularization constant and $\fmap : \Xset \to \R^h$ is the feature map. The following problem is defined:
\label{ProbSSVM}
 \begin{equation*}
 \label{EqSSVM}
  \minpcl{w, e_i, b}{\frac{1}{2} w^{T} w + \frac{\gamma}{2} \sum_{i=1}^{N} e_i^2}{y_i = w^{T}\fmapx{x_i} + b + e_i \sepfor i=1, \dotsc, N \eeq{.}}
 \end{equation*}
\end{myprob}

On the other hand, the unsupervised learning problem of Kernel PCA (\kpca{}) can also be obtained from the following variational principle~\cite{IEEE_Suykens,BookSuykens}, considering the dataset $\set{x_i}_{i=1}^{N}$ where $x_i \in \Xset$ and a feature map $\fmap : \Xset \to \R^h$.
\begin{myprob}[\kpca{}] For $\gamma > 0$, the problem is:
\label{ProbKPCA}
 \begin{equation*}
  \minpcl{w, e_i}{\frac{1}{2} w^{T} w - \frac{\gamma}{2} \sum_{i=1}^{N} e_i^2}{e_i = w^{T}\fmapx{x_i} \sepfor i=1, \dotsc, N \eeq{.}}
 \end{equation*}
\end{myprob}
In general, \ref{ProbKPCA} does not have a solution though it is of interest as we shall explain in the sequel. The first term is interpreted as a regularization term while the second term is present in order to maximize the variance. Hence, the solutions provide the  model $w^{T} \fmapx{x}$ for the principal components in feature space. Indeed, in~\cite{BookSuykens} it is shown that a solution to the optimality conditions has to satisfy $K \alpha = \prn{1/\gamma} \alpha$, where $K$ is the kernel matrix, $K_{ij} = \fmapx{x_i}^{T} \fmapx{x_j}$. This is only possible if $\gamma = 1/\lambda$, where $\lambda$ is an eigenvalue of $K$. Noticeably, a reweighted version can be used for formulating a spectral clustering problem~\cite{IEEE_Alzate}.

Let us denote the eigenvalues of $K$ of rank $r \leq N$ by $\eigval{\max} = \eigval{1} \geq \eigval{2} \geq \dots \geq \eigval{r}$ and its eigenvectors $\eigvec{1}, \dotsc,\eigvec{r}$. We shall now formulate a convex optimization problem for the different principal components, on a general feature space.
%%%%%%%%%%%%%%%%%%%%%%%%%%%%%%%%%%%%%%%%%%%%%

%%%%%%%%%%%%%%%%%%%%%%%%%%%%%%%%%%%%%%%%%%%%%
\section{Convex Formulation of \kpca{}}
\label{SecConvexOpt}

We will reformulate now the problem introduced in~\cite{IEEE_Suykens} for \kpca{}. Although the resulting models are the same, we get a convex problem with our formulation whose minimum coincides with the kernel principal components of the data for particular choices of the regularization parameter $\gamma$.

% Preliminaries.
For the sake of generality, we shall consider that the feature space is a separable Hilbert space $\prn{\Hsp, \inp{\cdot}{\cdot}}$, which can be infinite dimensional.
We assume also that $\Hsp$ is defined over the reals so that the inner product verifies $\inp{\phi}{\chi} = \inp{\chi}{\phi}$. Hence, the feature map is $\fmap : \Xset \to \Hsp$. For convenience, we shall use the bra-ket notation~\cite{Dirac}: an element of $\Hsp$ will be written $\ket{\phi}$ whereas an element of its Fr\'echet-Riesz dual is $\bra{\phi} \in \Hsp^* \sim \Hsp$. We shall consider a finite dimensional Hilbert space $\prn{\Esp,\inpe{\cdot}{\cdot}}$ of dimension $N$ with an orthonormal basis $\eb_1, \dotsc, \eb_N$ and the associated dual basis $\eb^*_1, \dots, \eb^*_N$, verifying $\eb^*_i\prn{\eb_j} = \delta_{i,j}$, for $i, j = 1, \dotsc, N$. For simplicity, the dual basis can be thought of as the canonical basis of row vectors, while the primal basis is given by the canonical basis of column vectors.
Let us define the linear operators $\Phi: \Esp \to \Hsp$ and $\Phi^*: \Hsp \to \Esp$  given by the finite sums $\Phi = \sum_{i=1}^{N} \ket{\fmapx{x_i}} \eb_i^*$ and $\Phi^* = \sum_{i=1}^{N} \eb_i \bra{\fmapx{x_i}}$.
Hence, we can introduce the linear operators $\Phi \circ \Phi^* : \Hsp \to \Hsp$, given by the covariance $\covard = \sum_{i=1}^{N} \ket{\fmapx{x_i}} \bra{\fmapx{x_i}}$, while the kernel operator $\kernelp : \Esp \to \Esp$ is given by $\kernelp = \sum_{i,j = 1}^{N} K\prn{x_j,x_i} \eb_j \eb_i^*$, where we have identified the space of linear operators with $\Hsp \times \Hsp^*$ and $\Esp \times \Esp^*$, respectively.
We shall show that these two operators have the same rank.
\begin{mylemma}
 The operators $\covard$ and $\kernelp$ are self-adjoint, positive semi-definite and have the same non-zero eigenvalues. They have the same rank $r$, i.e. they share $r \leq N$ non-zero eigenvalues $\eigval{1} \geq \eigval{2} \geq\dots \geq \eigval{r} > 0$.
\end{mylemma}
\begin{proof}
 Self-adjointness is a consequence of the definition of the inner products. The equivalence of the non-zero eigenvalues is proved as follows.
 \begin{enumerate*}[(i)]
  \item Let us suppose that $\prn{\kernelp} \eigveca{} = \eigval{} \eigveca{}$ with $\eigval{} \neq 0$. Acting on both sides with $\Phi$ gives: $\Phi \circ \prn{\kernelp} \eigveca{} = \prn{\covard} \circ \Phi \eigveca{} = \eigval{} \Phi \eigveca{}$. Hence, we have that $\Phi \eigveca{}$ is an eigenvector of $\covard$ with eigenvalue $\eigval{}$.
  \item Similarly, if we assume that $\prn{\covard} \ket{\eigvec{}} = \eigval{} \ket{\eigvec{}}$ with $\eigval{} \neq 0$. Then, acting on this equation with $\Phi^*$ gives $\prn{\kernelp} \circ \Phi^* \ket{\eigvec{}} = \eigval{} \Phi^* \ket{\eigvec{}}$. Therefore, $\Phi^* \ket{\eigvec{}}$ is an eigenvector of $\kernelp$ with eigenvalue $\eigval{}$.
 \end{enumerate*}
 This proves that $\kernelp$ and $\covard$ share the same non-zero eigenvalues.
\end{proof}

Incidentally, the zero eigenvalues do not in general match.
As a result, we can write in components the eigenvalue equation for the kernel matrix as $\sum_{j=1}^{N} K\prn{x_i,x_j} \eigveca{\ell}^j= \eigval{\ell} \eigveca{\ell}^i$, with $\eigveca{\ell}^i = \inp{\fmapx{x_i}}{\eigvec{\ell}}$, which are the components of the eigenvectors $\eigveca{\ell} = \sum_{i=1}^{N} \eigveca{\ell}^i \eb_i$ of the operator $\kernelp$.
Furthermore, if the eigenvectors  are normalized as  $\inp{\eigvec{\ell}}{\eigvec{\ell}} = 1$, then we have $\eigveca{k}^* \eigveca{k} = \eigval{k}$, for $k = 1, \dotsc, r \leq N$.
For the sake of simplicity, the case where the eigenvalues are degenerate will not be treated explicitly. It is however straightforward to consider it.

\begin{mylemma}
\label{LemSpecRep}
 We have the spectral representations $\kernelp = \sum_{\ell=1}^{r} \eigveca{\ell} \eigveca{\ell}^*$ and $\covard = \sum_{\ell=1}^{r} \eigval{\ell} \ket{\eigvec{\ell}} \bra{\eigvec{\ell}}$, while the identity over $\Hsp$ can be written as $I_\Hsp = \sum_{\ell=1}^{r} \ket{\eigvec{\ell}} \bra{\eigvec{\ell}} + \proj_0$, where $\proj_0$ is the projector on the null space of $\covard$.
\end{mylemma}
\begin{proof}
 This is consequence that the eigenvectors of self-adjoint operators form a basis of the Hilbert space~\cite{Berberian}.
\end{proof}

% Kernel PCA.
We can formulate now a convex optimization problem for the $\prn{k+1}$-th largest principal component of the kernel matrix.
Let us assume that the eigenvalues of $\covard$ are sorted in decreasing order, i.e. $\eigval{1} \geq \eigval{2} \geq \dots \geq \eigval{r}$, with the corresponding orthonormal eigenvectors $\ket{\eigvec{1}}, \ket{\eigvec{2}}, \dotsc, \ket{\eigvec{r}}$.
\begin{myprob}[Constrained Formulation of \kpca{}]
\label{ProbKPCAConv}
 For a fixed value of $k$, with $1 \leq k \leq r-1$, the primal problem is:
 \begin{equation}
  \label{EqKPCAConv}
  \minpcla{\ket{w}, e_i}{\frac{1}{2} \inp{w}{w} - \frac{\gamma}{2} \sum_{i=1}^{N} e_i^2}{e_i = \inp{w}{\fmapx{x_i}} \sepfor i=1, \dotsc, N \eeq{,} \\ \inp{w}{\eigvec{\ell}} = 0 \sepfor \ell=1, \dots, k \eeq{.}}
 \end{equation}
 When $k = 0$, the problem should be understood without constraints on $\ket{w}$, and thus \ref{ProbKPCA} is recovered over $\Hsp$.
\end{myprob}
Therefore, \ref{ProbKPCAConv} is similar to \ref{ProbKPCA} but with $k$ additional constraints that limit the feasible set for $\ket{w}$.
This problem is interesting from the methodological point of view because it specifies which optimization problem the computation of the kernel principal components is solving.
Indeed, the optimization is maximizing the variance while the model is simultaneously required to be regular. We state now a useful result concerning the existence of solutions.

\begin{mythm}[Convexity and Solutions of \kpca{}]
\label{TheoKPCA}
 \Ref{ProbKPCAConv} is convex for $0 \leq k \leq r-1$ iff $\gamma \leq 1/\eigval{k+1}$.
 In particular:
 \begin{enumerate}[(i)]
  \item If $\gamma < 1 / \eigval{k+1}$, the problem admits only the trivial solution $\ket{\opt{w}} = 0$ and  $e_i = 0$ for $i=1, \dots, N$.
  \item If $\gamma = 1 / \eigval{k+1}$, the problem admits as solutions $\ket{\opt{w}} \propto \ket{\eigvec{k+1}}$, $e_i \propto \inp{\eigvec{k+1}}{\fmapx{x_i}}$, for $i=1, \dots, N$.
  \item If $\gamma > 1/\eigval{k+1}$, the problem has no bounded solution.
 \end{enumerate}
\end{mythm}
\begin{proof}
 Substituting the first constraint in the objective function, we obtain the problem
 \begin{equation}
 \label{FiniteSum}
  \minpcl{\ket{w} \in \Hsp}{\frac{1}{2} \inpt{w}{I_\Hsp - \gamma \covard}{w}}{\inp{w}{\eigvec{\ell}} = 0 \sepfor \ell=1, \dots, k \eeq{.}}
 \end{equation}
 Then, it is clear from the spectral representation of $\covard$ given in \ref{LemSpecRep} that the objective function of \ref{FiniteSum} reduces to a finite sum. We can write $\ket{w} = \sum_{\ell=1}^{r} \ket{\eigvec{\ell}} \inp{\eigvec{\ell}}{w} + \proj_0 \ket{w}$, and, therefore, solving the whole set of constraints in the objective function \ref{FiniteSum}, the unconstrained problem reads
\begin{equation}
\label{FiniteSumDecom}
 \minp{\ket{w} \in \Hsp^\perp}{\frac{1}{2} \inp{\proj_0 w}{\proj_0 w} + \frac{1}{2} \sum_{\ell={k+1}}^{r}\prn{1 - \gamma \eigval{\ell}} \prn{\inp{\eigvec{\ell}}{w}}^2} \eeq{,}
\end{equation}
where we have defined $\Hsp = \Hsp^\perp + \spans{\ket{\eigvec{1}}, \dotsc, \ket{\eigvec{k}}}$.
The objective of \ref{FiniteSumDecom} is obviously convex iff $\gamma \leq 1/\eigval{k+1} \leq \dots \leq 1/\eigval{N}$. Moreover, we have $\proj_0 \ket{w} = 0$ since this is the only minimizer of the first term of \ref{FiniteSumDecom}. Finally, if $\gamma < 1/\eigval{k+1}$ the only solution is $\ket{w} = 0$. Otherwise, if $\gamma = 1/\eigval{k+1}$, the solutions form the vector subspace $\spans{\ket{\eigvec{k+1}}}$.
\end{proof}
The idea of \ref{TheoKPCA} is illustrated in \ref{FigSurfacesA,FigSurfacesB,FigSurfacesC,FigSurfacesD,FigSurfacesE} for a two-dimensional case. Moreover, \ref{FigSchemes} shows the different values for which the problem has a solution.
To put it in a nutshell, we are minimizing the objective of \ref{EqKPCAConv}, which is purely a sum of quadratic terms. If the first term dominates the second term in all directions, then the objective will always be greater or equal to $0$, and thus the minimum will be the trivial solution $\ket{w} = 0$ (\ref{FigSurfacesA}). If the dominant quadratic part of both terms is the same, we get a subspace of non-trivial solutions, all with objective $0$ (\ref{FigSurfacesB}). If the second term dominates the first one in some direction, then the objective is not lower bounded, so there is no bounded solution (\ref{FigSurfacesC}). Nevertheless, if we remove the dominant directions of the second term (the appropriate eigenvectors), we can recover the previous two cases (convex black curve in \ref{FigSurfacesC}).

\begin{mfiguretikz}{\label{FigSurfaces} Example of the objective function of \ref{ProbKPCAConv,ProbSKPCA} for two-dimensional data (with the constraints on $e_i$ substituted).
Regarding \kpca{}, in the case of \ref{FigSurfacesA} the problem has only one solution; in \ref{FigSurfacesB} the problem has a one-dimensional subspace of solutions given by $\eigvec{1}$. \Ref{FigSurfacesC,FigSurfacesD,FigSurfacesE} have no bounded solution; nevertheless, if we ignore the direction of the first eigenvector, \ref{FigSurfacesC} has one solution and \ref{FigSurfacesD} has a subspace of solutions given by $\eigvec{2}$, whereas \Ref{FigSurfacesE} is not bounded in any direction.
With respect to \skpca{}, in the case of \ref{FigSurfacesSA} the problem has only one solution; in \ref{FigSurfacesSB} the problem has no bounded solution since the objective goes to $- \infty$ in the direction of $\eigvec{1}$. If we ignore the direction of the first eigenvector, \ref{FigSurfacesSB,FigSurfacesSC} has one solution, the objective in \ref{FigSurfacesSD} goes to $- \infty$ in the direction of $\eigvec{2}$, and \ref{FigSurfacesSE} is not bounded in any direction.}
 \iftikz\input{./Figures/Tikz/SurfaceKPCA}\fi
 \iftikz\input{./Figures/Tikz/SurfaceSKPCA}\fi
 \tikzwidth{0.19\textwidth}
 \rotatebox{90}{\footnotesize\qquad\kpca{}}\quad%
 \subfloat[\label{FigSurfacesA}$\gamma < 1/\eigval{1}$.]{\includetikzf{\plotSurface{0.125}}}\hfill%
 \subfloat[\label{FigSurfacesB}$\gamma = 1/\eigval{1}$.]{\includetikzf{\plotSurface{0.25}}}\hfill%
 \subfloat[\label{FigSurfacesC}$1/\eigval{1} < \gamma < 1/\eigval{2}$.]{\includetikzf{\plotSurface{0.5}}}\hfill%
 \subfloat[\label{FigSurfacesD}$\gamma = 1/\eigval{2}$.]{\includetikzf{\plotSurface{1}}}\hfill%
 \subfloat[\label{FigSurfacesE}$\gamma > 1/\eigval{2}$.]{\includetikzf{\plotSurface{4}}}\\
 \rotatebox{90}{\footnotesize\qquad\skpca{}}\quad%
 \subfloat[\label{FigSurfacesSA}$\gamma < 1/\eigval{1}$.]{\includetikzf{\plotSurfaceS{0.125}}}\hfill%
 \subfloat[\label{FigSurfacesSB}$\gamma = 1/\eigval{1}$.]{\includetikzf{\plotSurfaceS{0.25}}}\hfill%
 \subfloat[\label{FigSurfacesSC}$1/\eigval{1} < \gamma < 1/\eigval{2}$.]{\includetikzf{\plotSurfaceS{0.5}}}\hfill%
 \subfloat[\label{FigSurfacesSD}$\gamma = 1/\eigval{2}$.]{\includetikzf{\plotSurfaceS{1}}}\hfill%
 \subfloat[\label{FigSurfacesSE}$\gamma > 1/\eigval{2}$.]{\includetikzf{\plotSurfaceS{4}}}
\end{mfiguretikz}

\begin{figure}[t]
\begin{center}
 \includetikzf{\input{./Figures/Tikz/SchemeV1A.tikz}}\\
 \includetikzf{\input{./Figures/Tikz/SchemeV1B.tikz}}
 \caption[Schemes of convexity.]{\label{FigSchemes}Schemes of the convexity for different values of $\gamma$, both for \kpca{} and \skpca{}. Legend: \showline{coltriv} trivial solution ($\ket{w} = 0$); \showline{colntriv} non-trivial solution; \showline{colnboun} absence of bounded solution.}
\end{center}
\end{figure}

As a matter of fact, the convex formulation given in \ref{ProbKPCAConv} does not provide any new algorithmic approach for computing the principal components. Rather, it is of methodological interest since it provides a framework which potentially allows to define extensions of the current definition.

\section{Semi-Supervised Learning}
\label{SecSemiSup}

A semi-supervised learning problem is defined here by starting from a set of pairs of data points and labels $\set{\prn{x_i,y_i}}_{i=1}^{N}$ where $y_i = \pm 1$ if the point $x_i \in \Xset$ is labeled in the class $+1$ or $-1$ respectively, and $y_i = 0$ if $x_i$ is unlabeled. Following the book~\cite{3603}, we make the smoothness assumption, i.e. if two points are in the same high-density region they should be similar.
Based on the intuition gained from \ref{ProbSSVM,ProbKPCAConv}, we propose to define a hybrid optimization problem, relying on the smoothing criterion provided by the eigenvector associated to the $\prn{k+1}$-th largest eigenvalue of the kernel. In particular, we will combine the negative error term of \kpca{} with an LS-SVM-like loss term, what results in an optimization problem which is convex for the appropriate regularization parameter $\gamma$ and imposing certain constraints.

Let us firstly define the projectors on the first $k$ eigenvectors
\begin{equation}
\label{Projectors}
 \Pi_k = \sum_{\ell=1}^{k} \ket{\eigvec{\ell}} \bra{\eigvec{\ell}} \sepeq
 P_k = \sum_{\ell=1}^{k} \eigveca{\ell} \eigveca{\ell}^* \eeq{,}
\end{equation}
with the notations $\eigveca{\ell}^i = \inp{\fmapx{x_i}}{\eigvec{\ell}}$ and $\inp{\eigvec{\ell}}{\eigvec{\ell}} = 1$. Let us note also that $P_k = \Phi^* \circ \Pi_k \circ \Phi$.
The proposed semi-supervised \kpca{} (\skpca{}) is defined as follows.
\begin{myprob}[\skpca{}]
\label{ProbSKPCA}
 For a chosen $k$ satisfying $1 \leq k \leq r-1$, the following problem is defined:
 \begin{equation}
 \label{EqSKPCA}
  \minpcla{\ket{w}, e_i}{\frac{1}{2} \inp{w}{w} - \frac{\gamma}{2} \sum_{i=1}^{N} e_i^2}{e_i = y_i + \inp{w}{\fmapx{x_i}} \sepfor i=1, \dotsc, N \eeq{,} \\ \inp{w}{\eigvec{\ell}} = 0 \sepfor \ell=1, \dots, k \eeq{.}}
 \end{equation}
 When $k = 0$, the constraints on $\ket{w}$ are removed.
\end{myprob}

In the absence of labels, \ref{ProbSKPCA} reduces to \ref{ProbKPCAConv}. The meaning of the problem can be explained as follows.
\begin{enumerate*}[(i)]
 \item For unlabeled points, \ref{ProbSKPCA} is unsupervised and the variational principle requires a smooth solution, i.e. minimizing $\inp{w}{w}$, and maximizing at the same time the variance as in \ref{ProbKPCAConv}.
 \item For a labeled point $x_i$ with $y_i = \pm 1$, the second term of \ref{ProbSKPCA} can be interpreted as a concave error term. It requires $\prn{y_i + \inp{w}{\fmapx{x_i}}}^2$ to be large, so a favourable case arises when $\inp{w}{\fmapx{x_i}}$ and $y_i$ have the same sign. If the regularization term dominates, i.e. if $|\inp{w}{\fmapx{x_i}}|$ is large enough, the model can still predict an opposite class for $x_i$. which is interesting if some data points are incorrectly labeled.
\end{enumerate*}

Notice that the model can only be evaluated, without using additional numerical approximations or solving again the optimization problem, over the $N$ given data used to pose and solve \ref{ProbSKPCA}, both the labeled and the unlabeled points. This evaluation, formulated as $\inp{w}{\fmapx{x_i}}$, is given in terms of the dual solution $\alpha$ by $\hat{y} = \sign\brq{\prn{K - P_k}\alpha}$.

Our second main result, \ref{TheoSKPCA}, states the existence of solutions to \ref{ProbSKPCA}, based on the following known lemma.
\begin{mylemma}
\label{LemSKPCA}
 The solution $\ket{\opt{w}}$ of \ref{ProbSKPCA} verifies $\ket{\opt{w}} \in \spans{\ket{\fmapx{x_1}}, \dotsc, \ket{\fmapx{x_N}}} = \spans{\ket{\eigvec{1}}, \dotsc, \ket{\eigvec{r}}}$, so that $\ket{\opt{w}}$ belongs to a subspace of dimension $r \leq N$.
\end{mylemma}
\begin{mythm}[Convexity and Solutions of \skpca{}]
\label{TheoSKPCA}
 \Ref{ProbSKPCA} is strongly convex iff  $\gamma < 1/\eigval{k+1}$, and
 \begin{enumerate}[(i)]
  \item If $\gamma < 1/\eigval{k+1}$, its unique solution is given by $\ket{\opt{w}} = \prn{I - \Pi_k} \sum_{j=1}^{N} \alpha_j \ket{\fmapx{x_j}}$ and $e_i = \alpha_i/\gamma$, for all $i = 1, \dots, N$, where $\alpha = \prn{\alpha_1, \dotsc, \alpha_n}^\tr$ is obtained by solving the linear system $\prn{I/\gamma - K + P_k } \alpha = y$.
  \item If $\gamma \geq 1/\eigval{k+1}$, it has no solution.
 \end{enumerate}
\end{mythm}
\begin{proof}
 Let us decompose $\Hsp = \Hsp^\perp + \spans{\ket{\eigvec{1}}, \dotsc, \ket{\eigvec{k}}}$ in order to eliminate the orthogonality constraints, so we restrict ourselves to $\ket{w}\in \Hsp^\perp$. The problem can be factorized by acting with $\proj_0$, i.e. the projector on $\ker{\covard}$. Indeed, we can do an orthogonal decomposition $\Hsp^\perp = \proj_0 \Hsp^\perp + \Hsp_{r-k}$ with $\Hsp_{r-k} =\spans{\ket{\eigvec{k+1}}, \dotsc, \ket{\eigvec{r}}}$, so that we can solve separately two problems defined on orthogonal spaces.
 Substituting the constraints for the $e_i$ values, the objective function can be split into two terms $T_1 + T_2$. Removing constant and zero terms, the functional $T_1: \proj_0 \Hsp^\perp \to \R$ is given by $T_1\prn{\proj_0 \ket{w}} = \frac{1}{2} \inp{\proj_0 w}{\proj_0 w} - \gamma \inp{\proj_0 w}{\proj_0 \sum_{i=1}^{N} y_i\fmapx{x_i}}$, and the function $T_2: \Hsp_{r-k} \to \R$ by $T_2\prn{\ket{w} - \proj_0 \ket{w}} = \frac{1}{2} \sum_{\ell={k+1}}^{r} \prn{1 - \gamma \eigval{\ell}} \prn{\inp{\eigvec{\ell}}{w}}^2 - \frac{\gamma}{2} \sum_{i=1}^{N} y_i^2 - \gamma \sum_{i=1}^{N} \sum_{\ell={k+1}}^{r} y_i \inp{w}{\eigvec{\ell}} \inp{\eigvec{\ell}}{\fmapx{x_i}}$.
 However, we know that $\proj_0 \sum_{i=1}^{N} y_i \ket{\fmapx{x_i}} = 0$ because of \ref{LemSKPCA}, so that the solution of the optimization problem satisfies ${P}_0 \ket{w} = 0$.
 As a result, the infinite dimensional optimization problem is reduced to the problem $\minp{\ket{w} \in \Hsp_{r-k}}{T_2\prn{\ket{w}}}$ with a quadratic objective function, which is convex iff $\gamma < 1/\eigval{k+1}$.
 In the latter case, by solving the necessary and sufficient optimality conditions $\prn{1-\gamma \eigval{\ell}} \inp{\eigvec{\ell}}{w} - \gamma \sum_{i=1}^{N} y_i \inp{\eigvec{\ell}}{\fmapx{x_i}} = 0$, for $\ell = k+1, \dotsc, r$, we obtain the primal solution $\ket{\opt{w}} = \sum_{\ell=k+1}^{r} \frac{\ket{\eigvec{\ell}} \bra{\eigvec{\ell}}}{1/\gamma - \eigval{\ell}} \sum_{i=1}^{N} y_i \ket{\fmapx{x_i}}$.
 The latter formula is not applicable in practice since it requires the computations of the eigenvectors of $\covard$.
 Hence, it is useful to study the Lagrange dual. The Lagrange function of the convex optimization problem \ref{EqSKPCA} is $\lagr = \frac{1}{2} \inp{w}{w} - \frac{\gamma}{2} \sum_{i=1}^{N} e_i^2 + \sum_{\ell=1}^{k} \beta_{\ell} \inp{\eigvec{\ell}}{w} \\ + \sum_{i=1}^{N} \alpha_{i} \prn{e_i - y_i - \inp{w}{\fmapx{x_i}}}$.
 The KKT conditions are:
 \begin{equation*}
 \left \{
 \begin{array}{rcl}
  \ket{w} &=& \gamma \sum_{i=1}^{N} \alpha_{i} \ket{\fmapx{x_i}} - \sum_{\ell=1}^{k} \beta_{k} \ket{\eigvec{\ell}} \eeq{,} \\
  e_i &=& \alpha_{i}/\gamma \sepfor i=1, \dotsc, N \eeq{,} \\
  \Pi_k \ket{w} &=& 0 \eeq{,} \\
  y_i &=& e_i - \inp{w}{\fmapx{x_i}} \sepfor i=1, \dotsc, N \eeq{,}
 \end{array}
 \right .
 \end{equation*}
 where $\Pi_k$ is defined in \ref{Projectors}. The multipliers $\beta_\ell$ are obtained by solving the orthogonality constraint $\Pi_k \ket{w} = 0$. By eliminating $\ket{w}$ in the last of the above equations, and using the definition of the projector $P_k$ of \ref{Projectors}, we obtain \ref{TheoSKPCA}.
\end{proof}

Hence, the convexity of the semi-supervised learning problem will depend on the number of constraints $k$ and the selection of $\gamma$, and its solution can be obtained by solving a linear system.
The whole procedure is summarized in \ref{AlgSKPCA}, and the values of $\gamma$ for which the problem is convex are represented in \ref{FigSchemes}.
The intuition of these results is similar to that of \kpca{}, and it is illustrated in \ref{FigSurfacesSA,FigSurfacesSB,FigSurfacesSC,FigSurfacesSD,FigSurfacesSE}. Basically, if the first term in the objective function of \ref{EqSKPCA} dominates the second term, we obtain a solution (in this case it is not trivial precisely because of the information added by the labels $y$, see \ref{FigSurfacesSA}). If the second term dominates the first one, the problem is not bounded (\ref{FigSurfacesSC}). When both dominant parts are equal, then the problem is not bounded but the solution is in the direction of the largest (active) eigenvector (\ref{FigSurfacesSB}).

\begin{malgorithm}{\label{AlgSKPCA} Algorithm of \skpca{}.}
 \begin{algorithmic}[1]
  \REQUIRE {\quad} \\
   \algreq Data $\set{\prn{x_i,y_i}}_{i=1}^{N}$ where $y_i \in \set{-1,0,1}$ \algstop \\
   \algreq Number of constraints $0 \le k < N$ \algstop \\
   \algreq Regularization parameter $\gamma$ \algstop
  \ENSURE {\quad} \\
   \algreq Predicted labels $\hat{y}$ \algstop
  \STATE Build kernel matrix $K \in \R^{N \times N}$ \algstop
  \STATE Compute $k$ largest eigenvalues $\eigval{1} \ge \cdots \ge \eigval{k}$ with eigenvectors $\eigvecv{1}, \dotsc, \eigvecv{k}$ \algstop
  \IF{$\gamma \ge 1/\eigval{k + 1}$}
   \RETURN Error: Unbounded problem \algstop
  \ENDIF
  \IF{$k > 0$}
   \STATE $K' \gets K - \sum_{i = 1}^{k} \eigval{i} \eigvecv{i} \eigvecv{i}^\tr$, the projected kernel matrix \algstop
  \ELSE
   \STATE $K' \gets K$ \algstop
  \ENDIF
  \STATE Solve the system $\prn{I / \gamma - K'}\alpha = y$ for the vector $\alpha \in \R^N$ \algstop
  \RETURN $\hat{y} = \sign\brq{K' \alpha}$ \algstop
 \end{algorithmic}
\end{malgorithm}

In practice, a good choice for \ref{ProbSSVM} is $k = 1$ since it is well-known in spectral clustering that the largest eigenvalue does not incorporate information for connected graphs~\cite{Luxburg}.
%%%%%%%%%%%%%%%%%%%%%%%%%%%%%%%%%%%%%%%%%%%%%

%%%%%%%%%%%%%%%%%%%%%%%%%%%%%%%%%%%%%%%%%%%%%
\section{Numerical Examples}
\label{SecExamples}

We will show experimentally how the proposed \skpca{} can deal with semi-supervised classification problems, particularly well when the number of labeled data is limited.
% Data.
The proposed model is tested over the nine binary classification datasets of \ref{TabDatasets}: six from the KEEL-dataset repository~\cite{alcala2010keel}, two from the UCI one~\cite{Lichman:2013}, and a synthetic example conformed by four Gaussian clusters, two of each class, with separation $2 \sigma$ between classes and $2.5 \sigma$ between clusters of the same class.
% Models.
We compare four models:
\begin{enumerate*}[(i)]
 \item \skpca{} with $k = 0$ (\skpcak{0}), i.e., the one corresponding to \ref{ProbSKPCA} without orthogonality constraints, which is convex for $0 \le \gamma < 1/\eigval{1}$;
 \item \skpca{} with $k = 1$ (\skpcak{1}), i.e. with one constraint, and hence convex for $0 \le \gamma < 1/\eigval{2}$;
 \item a simple semi-supervised variant of LS-SVM (\ssvm{}), where the target for the unlabeled patterns is $0$ (this approach shares similarities with the works of~\cite{NIPS2003_2506,SemiEigenMaps}); and
 \item LS-SVM trained only over the labeled subsample of the data (\svm{}).
\end{enumerate*}
For the last two models, we will omit the bias term $b$ since it does not improve the results with the small number of supervised patterns considered here.
% Framework.
The number of labeled data is varied as $1\%$, $2\%$, $5\%$, and $10\%$ the total number of patterns. Each experiment is repeated ten times to average the results. As a measure of the performance, we use the accuracy over the unlabeled data.
% Parameter.
Regarding the parameter $\gamma$, its selection can be crucial for the performance, moreover, it is intrinsically difficult in semi-supervised learning since the amount of labeled data is very limited. Since this problem affects both the new \skpca{} and \ssvm{}/\svm{}, we consider two different set-ups. The first one is to select the best $\gamma$ parameter for each model looking at the test error, simulating the existence of a perfect validation criteria. The second one, more realistic, is to select an intermediate $\gamma$ value, which for the case of \skpca{} is at the middle in logarithmic scale of the interval in which the model is convex. For \ssvm{} and \svm{}, $\gamma$ is fixed as $\gamma = 10 d /N$ and $\gamma = 100 d /N$ respectively, i.e. a value of $10$ and $100$ but normalized.
With respect to the kernel, we use a Gaussian kernel with the bandwidth fixed to the median of the Euclidean distances between the points.

\begin{mtables}{\label{TabDatasets}Description of the Datasets}
 \begin{tabular}{*5c}
  \toprule
   \tabformathead{Dataset} & \tabformathead{Source} & \tabformathead{\#Patterns ($N$)} & \tabformathead{\#Features ($d$)} & \tabformathead{Majority Class (\%)} \\
  \midrule
\texttt{australian} & KEEL-dataset & $690$ & $14$ & $55.5$\% \\
\texttt{breastcancer} & KEEL-dataset & $683$ & $10$ & $65.0$\% \\
\texttt{diabetes} & UCI & $768$ & $8$ & $65.1$\% \\
\texttt{heart} & KEEL-dataset & $270$ & $13$ & $55.6$\% \\
\texttt{iris} & UCI & $150$ & $4$ & $66.7$\% \\
\texttt{monk-2} & KEEL-dataset & $432$ & $6$ & $52.8$\% \\
\texttt{pima} & KEEL-dataset & $768$ & $8$ & $65.1$\% \\
\texttt{sonar} & KEEL-dataset & $208$ & $60$ & $53.4$\% \\
\texttt{synth} & Synthetic & $400$ & $2$ & $50.0$\% \\
  \bottomrule
 \end{tabular}
\end{mtables}

% Results.
The results are shown in \ref{TabResults}, where \skpcak{0} is omitted for being systematically worse than \skpcak{1}, confirming our previous guess that $k = 1$ should be more informative than $k = 0$. This table includes the number of labeled patterns, and the mean and standard deviations of the accuracies of \skpcak{1}, \ssvm{} and \svm{}, both for the best and heuristic $\gamma$. The colours show visually the ranking at each block.
We can see how the proposed \skpcak{1} improves the results of \ssvmsvm{} in almost all the datasets when the ratio of labeled data is small. This was to be expected, as when the amount of labels increases, both \ssvmsvm{} tend to the standard LS-SVM, whereas our proposed model is best suited for semi-supervised learning with very few training labels. About the datasets of \texttt{pima} and \texttt{diabetes}, these two are classification datasets that probably do not satisfy the classical requirements for semi-supervised learning of presenting a low-density region between the classes. This can explain the low performance of the compared models. In the case of the \texttt{sonar} dataset, the number of labels is only two for the first set-up, probably too low for the complexity of the problem. Comparing \ssvm{} and \svm{}, we can see that the naive semi-supervised approach is better in more than half of the datasets for the smallest number of labels, supporting that this simple approach takes advantage of the unlabeled data.
Regarding the two criteria for selecting $\gamma$, the conclusions are the same for both approaches, \skpcak{1} being better than \ssvmsvm{} for small amounts of labels.
Moreover, in average the difference between the results using the ground-truth $\gamma$ and the heuristic $\gamma$ is smaller for \skpcak{1} than for \ssvmsvm{}.
As an example of how the accuracy changes across the different $\gamma$ values, \ref{FigResults} shows the results over two datasets for the four models, using only $1\%$ of the labels.

\begin{mtables}{\label{TabResults}Experimental Results}
\setlength{\tabcolsep}{0pt}
 \begin{tabular}{c@{\enspace}c@{\enspace}*3c@{\enspace}*3c}
  \toprule
   \multirow{2}{*}{\rotatebox[origin=c]{90}{\tabformathead{Data}}} & \multirow{2}{*}{\tabformathead{Labs.}} & \multicolumn{3}{c}{\tabformathead{Accuracy (\%) - Best $\gamma$}} & \multicolumn{3}{c}{\tabformathead{Accuracy (\%) - Fixed $\gamma$}}\\ \cmidrule(lr){3-5} \cmidrule(lr){6-8}
   & & \tabformathead{\skpcak{1}} & \tabformathead{\ssvm{}} & \tabformathead{\svm{}}  & \tabformathead{\skpcak{1}} & \tabformathead{\ssvm{}} & \tabformathead{\svm{}} \\
  \midrule
\multirow{4}{*}{\rotatebox[origin=c]{90}{\texttt{austral.}}}
 & 7 & \formata{$83.0\pm0.7$} & \formatb{$72.2\pm6.0$} & \formatc{$71.3\pm10.4$} & \formata{$80.8\pm2.9$} & \formatb{$71.5\pm5.7$} & \formatc{$65.8\pm11.7$}\\
 & 14 & \formata{$83.0\pm0.7$} & \formatc{$76.1\pm4.2$} & \formatb{$77.3\pm6.0$} & \formata{$81.7\pm2.1$} & \formatc{$75.1\pm4.4$} & \formatb{$76.6\pm6.5$}\\
 & 35 & \formata{$84.1\pm0.5$} & \formatc{$79.3\pm3.1$} & \formatb{$81.4\pm4.1$} & \formata{$83.9\pm1.1$} & \formatc{$79.3\pm3.1$} & \formatb{$81.4\pm4.8$}\\
 & 69 & \formatb{$84.0\pm1.5$} & \formatc{$83.7\pm1.6$} & \formata{$85.4\pm1.0$} & \formatb{$83.7\pm1.6$} & \formatc{$83.2\pm1.6$} & \formata{$85.4\pm1.0$}\\
\midrule
\multirow{4}{*}{\rotatebox[origin=c]{90}{\texttt{breastc.}}}
 & 7 & \formata{$95.5\pm0.2$} & \formatc{$90.1\pm4.6$} & \formatb{$93.6\pm2.7$} & \formata{$95.2\pm0.7$} & \formatb{$86.6\pm3.6$} & \formatc{$85.4\pm10.3$}\\
 & 14 & \formata{$95.4\pm0.1$} & \formatc{$94.1\pm5.1$} & \formatb{$94.9\pm3.0$} & \formata{$94.8\pm0.7$} & \formatc{$91.1\pm4.5$} & \formatb{$94.1\pm6.0$}\\
 & 34 & \formatc{$95.5\pm0.2$} & \formatb{$95.7\pm1.8$} & \formata{$96.0\pm0.7$} & \formatb{$95.1\pm0.5$} & \formatc{$93.5\pm2.1$} & \formata{$95.9\pm0.6$}\\
 & 68 & \formatc{$95.5\pm0.2$} & \formatb{$96.0\pm0.7$} & \formata{$96.5\pm0.5$} & \formatc{$95.2\pm0.4$} & \formatb{$95.6\pm0.9$} & \formata{$96.3\pm0.5$}\\
\midrule
\multirow{4}{*}{\rotatebox[origin=c]{90}{\texttt{diabetes}}}
 & 8 & \formata{$68.9\pm1.1$} & \formatc{$66.6\pm7.9$} & \formatb{$67.3\pm6.6$} & \formatc{$63.7\pm5.6$} & \formata{$65.7\pm7.1$} & \formatb{$63.8\pm10.5$}\\
 & 15 & \formatc{$69.4\pm0.2$} & \formatb{$69.5\pm3.6$} & \formata{$70.3\pm4.0$} & \formatc{$68.2\pm2.8$} & \formatb{$69.2\pm3.4$} & \formata{$69.6\pm3.4$}\\
 & 39 & \formatc{$69.5\pm0.4$} & \formatb{$72.1\pm2.5$} & \formata{$72.9\pm2.6$} & \formatc{$69.4\pm1.6$} & \formatb{$71.9\pm2.2$} & \formata{$72.7\pm3.0$}\\
 & 77 & \formatc{$69.6\pm1.6$} & \formatb{$73.9\pm2.2$} & \formata{$74.4\pm1.9$} & \formatc{$69.3\pm1.8$} & \formatb{$73.9\pm2.2$} & \formata{$74.4\pm1.9$}\\
\midrule
\multirow{4}{*}{\rotatebox[origin=c]{90}{\texttt{heart}}}
 & 3 & \formata{$73.8\pm16.1$} & \formatb{$60.0\pm7.1$} & \formatc{$56.0\pm9.3$} & \formata{$69.3\pm12.0$} & \formatb{$58.2\pm5.9$} & \formatc{$52.3\pm5.7$}\\
 & 6 & \formata{$74.5\pm17.9$} & \formatb{$65.5\pm9.4$} & \formatc{$65.1\pm14.2$} & \formata{$71.5\pm13.6$} & \formatc{$63.6\pm8.3$} & \formatb{$63.6\pm12.2$}\\
 & 14 & \formata{$80.6\pm1.6$} & \formatc{$68.4\pm6.2$} & \formatb{$70.7\pm7.8$} & \formata{$79.3\pm3.4$} & \formatc{$68.0\pm5.2$} & \formatb{$70.3\pm9.0$}\\
 & 27 & \formata{$80.7\pm1.8$} & \formatc{$80.0\pm2.1$} & \formatb{$80.7\pm1.7$} & \formata{$80.4\pm1.8$} & \formatc{$76.9\pm2.8$} & \formatb{$80.0\pm2.4$}\\
\midrule
\multirow{4}{*}{\rotatebox[origin=c]{90}{\texttt{iris}}}
 & 2 & \formata{$91.4\pm14.9$} & \formatb{$73.3\pm17.0$} & \formatc{$72.6\pm25.9$} & \formata{$91.1\pm13.3$} & \formatb{$73.1\pm13.5$} & \formatc{$72.6\pm25.9$}\\
 & 3 & \formata{$93.6\pm2.6$} & \formatc{$87.1\pm11.4$} & \formatb{$89.5\pm16.3$} & \formata{$92.8\pm5.3$} & \formatc{$86.4\pm10.7$} & \formatb{$89.5\pm16.3$}\\
 & 8 & \formata{$94.8\pm3.0$} & \formatc{$90.3\pm11.7$} & \formatb{$93.0\pm14.7$} & \formata{$94.8\pm3.0$} & \formatc{$90.3\pm11.7$} & \formatb{$92.7\pm14.6$}\\
 & 15 & \formatc{$95.8\pm3.7$} & \formatb{$98.1\pm2.8$} & \formata{$99.9\pm0.3$} & \formatc{$95.3\pm2.4$} & \formatb{$98.1\pm2.8$} & \formata{$99.9\pm0.3$}\\
\midrule
\multirow{4}{*}{\rotatebox[origin=c]{90}{\texttt{monk-2}}}
 & 4 & \formata{$69.2\pm4.9$} & \formatb{$68.0\pm5.4$} & \formatc{$66.6\pm8.1$} & \formata{$68.8\pm5.0$} & \formatb{$68.0\pm5.5$} & \formatc{$59.4\pm10.9$}\\
 & 9 & \formatb{$70.9\pm5.7$} & \formatc{$70.7\pm6.8$} & \formata{$73.9\pm9.1$} & \formata{$70.7\pm5.8$} & \formatb{$69.9\pm8.5$} & \formatc{$67.6\pm12.4$}\\
 & 22 & \formatc{$75.3\pm3.6$} & \formatb{$78.2\pm3.5$} & \formata{$79.8\pm3.0$} & \formatc{$75.0\pm3.8$} & \formatb{$76.3\pm4.5$} & \formata{$76.9\pm3.3$}\\
 & 43 & \formatc{$79.3\pm2.3$} & \formatb{$84.9\pm3.2$} & \formata{$92.7\pm2.2$} & \formatc{$79.0\pm2.3$} & \formatb{$82.2\pm3.2$} & \formata{$84.2\pm1.8$}\\
\midrule
\multirow{4}{*}{\rotatebox[origin=c]{90}{\texttt{pima}}}
 & 8 & \formata{$65.4\pm12.3$} & \formatc{$65.0\pm5.6$} & \formatb{$65.3\pm3.9$} & \formatb{$62.8\pm8.9$} & \formata{$63.6\pm5.4$} & \formatc{$62.0\pm8.1$}\\
 & 15 & \formata{$69.1\pm0.9$} & \formatb{$68.1\pm2.5$} & \formatc{$67.9\pm2.5$} & \formatc{$66.1\pm7.2$} & \formatb{$67.5\pm3.5$} & \formata{$67.7\pm2.7$}\\
 & 39 & \formatc{$69.3\pm0.4$} & \formatb{$71.6\pm3.1$} & \formata{$72.2\pm2.7$} & \formatc{$68.4\pm1.8$} & \formatb{$71.3\pm2.9$} & \formata{$71.7\pm3.0$}\\
 & 77 & \formatc{$69.9\pm0.8$} & \formatb{$73.9\pm2.2$} & \formata{$74.1\pm2.4$} & \formatc{$69.6\pm0.9$} & \formatb{$73.9\pm2.2$} & \formata{$74.1\pm2.1$}\\
\midrule
\multirow{4}{*}{\rotatebox[origin=c]{90}{\texttt{sonar}}}
 & 2 & \formatb{$52.7\pm7.8$} & \formata{$55.7\pm2.5$} & \formatc{$50.0\pm7.2$} & \formatb{$52.1\pm7.7$} & \formata{$55.1\pm3.4$} & \formatc{$50.0\pm7.2$}\\
 & 4 & \formata{$58.3\pm7.4$} & \formatb{$56.7\pm5.1$} & \formatc{$53.9\pm6.3$} & \formata{$57.1\pm6.8$} & \formatb{$54.7\pm3.8$} & \formatc{$53.0\pm5.9$}\\
 & 11 & \formatc{$65.0\pm4.5$} & \formatb{$65.2\pm5.6$} & \formata{$65.4\pm7.5$} & \formatb{$64.0\pm3.6$} & \formatc{$61.8\pm5.4$} & \formata{$65.1\pm7.6$}\\
 & 21 & \formatc{$70.2\pm5.5$} & \formatb{$70.3\pm5.8$} & \formata{$71.6\pm5.0$} & \formatb{$69.0\pm4.9$} & \formatc{$66.6\pm4.7$} & \formata{$71.5\pm4.9$}\\
\midrule
\multirow{4}{*}{\rotatebox[origin=c]{90}{\texttt{synth}}}
 & 4 & \formata{$93.1\pm5.6$} & \formatc{$86.7\pm7.6$} & \formatb{$87.6\pm9.7$} & \formata{$92.7\pm6.1$} & \formatb{$86.2\pm8.9$} & \formatc{$63.2\pm21.6$}\\
 & 8 & \formatb{$96.1\pm2.4$} & \formatc{$92.9\pm5.0$} & \formata{$96.2\pm2.1$} & \formata{$96.0\pm2.6$} & \formatb{$91.5\pm9.4$} & \formatc{$85.8\pm14.5$}\\
 & 20 & \formata{$96.9\pm0.9$} & \formatc{$95.1\pm1.8$} & \formatb{$96.6\pm1.4$} & \formata{$96.8\pm0.9$} & \formatc{$94.2\pm3.1$} & \formatb{$94.7\pm3.2$}\\
 & 40 & \formata{$97.6\pm0.5$} & \formatc{$96.5\pm1.5$} & \formatb{$97.3\pm0.7$} & \formata{$97.4\pm0.5$} & \formatc{$96.4\pm1.8$} & \formatb{$96.9\pm0.9$}\\
  \bottomrule
 \end{tabular}
\end{mtables}

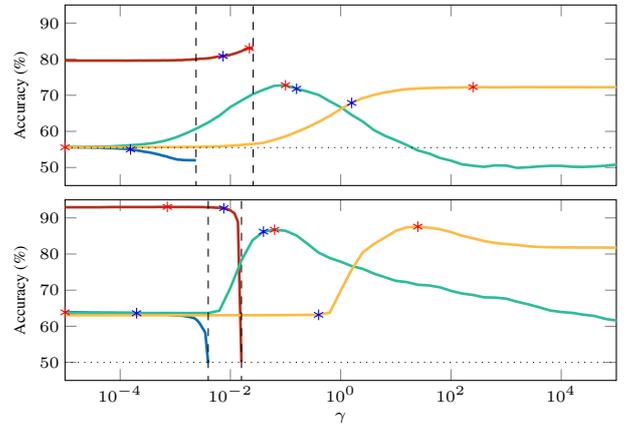
\begin{mfiguretikzs}{\label{FigResults}Mean accuracy with only $1\%$ of labels for \texttt{australian} (top) and \texttt{synth} (bottom). The black dashed lines indicate the limits of the convexity intervals ($1/\eigval{1}$ and $1/\eigval{2}$), and the dotted one the accuracy of the baseline error. The red asterisks \showaster{red} mark the best $\gamma$ for each model, and the blue asterisks \showaster{blue} the intermediate heuristic $\gamma$. Legend: \showline{NavyBlue}~\skpcak{0}; \showline{BrickRed}~\skpcak{1}; \showline{SeaGreen}~\ssvm{}; \showline{Dandelion}~\svm{}.}
 \tikzwidth{\textwidth}
 \includetikz{DemoSSDatasetE-australian-01}\\%
 \includetikz{DemoSSDatasetE-synth-01}
\end{mfiguretikzs}
%%%%%%%%%%%%%%%%%%%%%%%%%%%%%%%%%%%%%%%%%%%%%

%%%%%%%%%%%%%%%%%%%%%%%%%%%%%%%%%%%%%%%%%%%%%
\section{Conclusions}
\label{SecConcl}

Convex optimization is often desirable in machine learning. Supervised problems, as for instance SVMs for regression or classification, heavily rely on convex optimization problems given by a sum or convex combination of a convex regularization term and a convex loss function. The trade-off between smoothness of the model and fitting the data is usually given by a positive regularization constant.
Our first result of theoretical interest is the definition of a primal convex optimization problem for KPCA, including the possibility of an infinite dimensional feature space. This latter theoretical feature is desirable if, for instance, the Gaussian kernel is used for Kernel PCA. Besides, strong duality allows to consider a dual problem with the same solution, in analogy with SVMs.
Motivated by the introduction of a convex formulation of Kernel PCA, we have defined a new semi-supervised classification problem which can be interpreted as the minimization of the sum of a convex regularization term and a concave loss function. Although the loss function is concave, the convexity of the objective function is insured by the appropriate choice of constraints and trade-off parameter.
Our approach was illustrated by a series of numerical experiments on artificial and real data. The method, called \skpca{}, defined in this paper, was compared to a more conventional semi-supervised Least Squares SVM (LS-SVM) and to LS-SVM trained only over labeled data, leading to a better classification accuracy for a small number of labels.

We have proposed here a new type of convex optimization problems for machine learning. The upshot is that various further studies are now possible by choosing different concave loss functions in supervised problems for classification and regression, provided that the trade-off parameter is appropriately chosen to make sure that the objective function is convex, or some additional constraints are imposed.
%%%%%%%%%%%%%%%%%%%%%%%%%%%%%%%%%%%%%%%%%%%%%

%%%%%%%%%%%%%%%%%%%%%%%%%%%%%%%%%%%%%%%%%%%%%
\section*{Acknowledgments}

\begin{footnotesize}
 The authors would like to thank the following organizations.
 \begin{itemize*}
  \item EU: The research leading to these results has received funding from the European Research Council under the European Union's Seventh Framework Programme (FP7/2007-2013) / ERC AdG A-DATADRIVE-B (290923). This paper reflects only the authors' views, the Union is not liable for any use that may be made of the contained information.
  \item Research Council KUL: GOA/10/09 MaNet, CoE PFV/10/002 (OPTEC), BIL12/11T; PhD/Postdoc grants.
  \item Flemish Government:
  \begin{itemize*}
   \item FWO: G.0377.12 (Structured systems), G.088114N (Tensor based data similarity); PhD/Postdoc grants.
   \item IWT: SBO POM (100031); PhD/Postdoc grants.
  \end{itemize*}
  \item iMinds Medical Information Technologies SBO 2014.
  \item Belgian Federal Science Policy Office: IUAP P7/19 (DYSCO, Dynamical systems, control and optimization, 2012-2017).
 \end{itemize*}
\end{footnotesize}
%%%%%%%%%%%%%%%%%%%%%%%%%%%%%%%%%%%%%%%%%%%%%

%%%%%%%%%%%%%%%%%%%%%%%%%%%%%%%%%%%%%%%%%%%%%
\bibliographystyle{unsrt}
\bibliography{References}
%%%%%%%%%%%%%%%%%%%%%%%%%%%%%%%%%%%%%%%%%%%%%

\end{document}